\providecommand{\proofmode}{append}

\documentclass[sigconf,nonacm]{aamas}

\usepackage{balance}
\usepackage{tikz}
\usepackage{pgfplots}
\pgfplotsset{compat=1.18}
\usetikzlibrary{arrows.meta,positioning}
\usepackage[appendix=\proofmode,bibliography=common]{apxproof}
\renewcommand{\appendixprelim}{\clearpage}

\theoremstyle{acmplain}
\newtheoremrep{theorem}{Theorem}%
\newtheoremrep{proposition}[theorem]{Proposition}
\newtheoremrep{lemma}[theorem]{Lemma}
\newtheoremrep{corollary}[theorem]{Corollary}
\newtheoremrep{example}{Example}%

\makeatletter
\newcommand{\restoresection}{\@ifundefined{axp@oldsection}{}{\let\section\axp@oldsection}}
\makeatother

\setcopyright{none}
\acmDOI{}
\acmISBN{}

\newcommand{\Load}{\mathrm{Load}}
\newcommand{\Choose}{\mathrm{Choose}}
\newcommand{\op}{\mathtt{op}}
\newcommand{\rulesep}{\mathrel{\|}}
\newcommand{\sstate}{\mathrm{state}}
\newcommand{\smode}{\mathrm{mode}}
\newcommand{\sreg}{\mathrm{reg}}

\title{Learning How to Search for Plans with Exponentially Less Space}

\author{Dominik Drexler}
\affiliation{
  \institution{Link{\"o}ping University}
  \city{Linköping}
  \country{Sweden}
}
\email{dominik.drexler@liu.se}

\author{Simon St{\aa}hlberg}
\affiliation{
  \institution{RWTH Aachen University}
  \city{Aachen}
  \country{Germany}
}
\email{simon.stahlberg@gmail.com}

\author{Markus Fritzsche}
\affiliation{
  \institution{Link{\"o}ping University}
  \city{Linköping}
  \country{Sweden}
}
\email{markus.fritzsche@liu.se}

\author{Blai Bonet}
\affiliation{
  \institution{Universidad Simón Bolívar}
  \city{Caracas}
  \country{Venezuela}
}
\email{bonetblai@gmail.com}

\begin{abstract}
Heuristic search for a plan can store exponentially many states, even when its heuristic is almost perfect.
We instead learn \emph{search control}, one specification per domain, written as an indexical policy: a generalized policy with registers that hold objects and modes that sequence its rules.
We add the \emph{choose rule}, which loads an object into a register and marks a backtracking point, where one candidate suffices; every other rule must work for all of its outcomes and needs no search.
Our main result is that structural termination, which rules out infinite executions, also bounds every execution by a polynomial in the number of objects.
A depth-first procedure then finds a plan in polynomial space, however large the state space, with no list of visited states.
The cost is time, exponential only in the \emph{choice depth}, the number of real choices along an execution.
Any class that such a policy solves therefore lies in NP, and in P at constant choice depth.
We learn these policies with a language model in a counterexample-guided loop that certifies termination, verifies the training tasks, and keeps the choice depth small.
With the learned policies, the procedure solves 1,709 of 1,890 test tasks of the IPC 2023 Learning Track and the Autoscale Agile suite, more than LAMA, BFWS, and Levitron, and most of them within one second and 100~MiB.
\end{abstract}

\keywords{Generalized Planning, Indexical Policies, Search, Policy Learning}

\newcommand{\tup}[1]{\ensuremath{\langle #1 \rangle}}
\newcommand{\set}[1]{\ensuremath{\{ #1 \}}}

\begin{document}

\maketitle

\pagestyle{plain}
\thispagestyle{plain}
\raggedbottom

\section{Introduction}

A classical planning task asks for a plan, a sequence of actions that transforms an initial state into a state that satisfies a goal.
Heuristic search is the standard way to find one, but even an almost perfect heuristic can leave an exponential amount of search~\cite{helmert-roeger-aaai2008}, and planners such as LAMA and BFWS store every state they generate, which can be exponentially many~\cite{richter-westphal-jair2010,lipovetzky-geffner-aaai2017}.

This paper is about \emph{search control}: a specification that tells a search procedure which paths in the state space are worth exploring.
We seek search control for a fixed domain and every task over it, whatever its objects, initial state, and goal, which is the setting of generalized planning~\cite{jimenez-et-al-ker2019}, and we express it as a policy of exactly that kind.
What differs is the job: a generalized policy \emph{decides} a plan, one action at a time and without search, whereas our search control \emph{guides a search} for a plan.

Our rule language is that of indexical policies~\cite{bonet-et-al-icaps2024}, which extend the Boolean and numerical features of standard policies with registers that store objects selected during execution and with operating modes that regulate the control flow.
In the Ferry domain, for example, a register holds the next car to be moved, while the modes cycle through phases for selecting, boarding, crossing, and unloading; a register thus names an object by its role, as in deictic representations~\cite{chapman-aimag1989,agre-chapman-ras1990}.
Such a policy has a strong solution concept: every transition that it allows must keep the goal reachable, so an agent may take any of them and never needs to backtrack~\cite{bonet-geffner-ijcai2018,bonet-et-al-aaai2019,frances-et-al-aaai2021}.
Recent work learns such policies with language models and certifies that they are \emph{structurally terminating}, which rules out infinite executions~\cite{anonymous-unpublished2026,anonymous-unpublished2026b}.
Deciding without search has a price: such a policy evaluates cheap features along a short execution, so it decides plan existence in polynomial time, and it cannot solve all solvable tasks of an NP-hard class, unless P${}={}$NP (Section~\ref{subsec:completeness}).

We remove this obstacle with one new rule, the \emph{choose rule}.
An ordinary \emph{load} rule stores an object in a register and, like every ordinary rule, must work for every object that it may store; a choose rule stores one too, but it is a backtracking point, where only one candidate has to work and the search finds it.
In Sokoban, a legal push can leave a box permanently misplaced, so a load rule over two pushable boxes is wrong while a choose rule over the same two is correct: the second box can be tried once the first fails.
Every other rule keeps its ordinary meaning, so a policy states, rule by rule, where search is needed and stays search-free everywhere else.
A policy and a task thus define an AND/OR graph over augmented states: ordinary rules give $\mathrm{AND}$ vertices, which need every outcome to succeed, and choose rules give $\mathrm{OR}$ vertices, which need one.
The policy solves the task when its initial vertex is solved, a least fixed point of this graph, and the \emph{scout}, the procedure that executes the policy, is a depth-first search that tries the candidates of each choose rule in turn and backtracks on failure.

Structural termination makes the graph acyclic, even across all candidates of the choose rules, so the scout needs no set of visited states, yet the path that it stores could still be long.
Our main result is that structural termination also bounds the length of every trajectory by a polynomial in the number of objects of the task.
The scout therefore finds a plan in polynomial space, however large the state space is, whereas explicit state-space search may need exponential space.
The price is paid in time, as the scout may try up to $B^k$ sequences of choices, where the \emph{choice depth} $k$ counts the choices with more than one candidate along a trajectory and the \emph{choice width} $B$ bounds how many candidates each one has.

The same bound delimits which classes these policies can cover.
A policy that solves a whole class turns the scout into a decision procedure, and since its plans have polynomial length, the class lies in NP, and in P if the choice depth is constant.
Towers of Hanoi is thus out of reach, since its shortest plans have exponential length~\cite{bonet-geffner-jair2024}, and so is Sokoban, unless NP${}={}$PSPACE~\cite{culberson-tr1997}.

We learn these policies in the counterexample-guided loop of previous work~\cite{anonymous-unpublished2026,anonymous-unpublished2026b}, in which a language model proposes policies, an adversary proposes tasks on which they fail, and the model is asked to keep the choice depth small.
We also extend the feature language with queries of relational algebra, and hence with predicates of any arity, which description-logic features cannot express, so previous work reformulated such domains~\cite{anonymous-unpublished2026}, and with arithmetic on numerical features.
We evaluate the learned policies on the IPC 2023 Learning Track test set and the Autoscale Agile suite.
With them, the scout solves 1,709 of 1,890 test tasks, more than LAMA, BFWS, and Levitron, and most of them within one second and 100~MiB, without storing any visited states.

\section{Related Work}

\paragraph{Search control.}
Control knowledge has long been used to prune a planner's search: TLPlan and TALplanner check hand-written temporal-logic formulas along the paths of a forward search and discard every path that violates one~\cite{bacchus-kabanza-aij2000,kvarnstrom-doherty-amai2000}.
Hierarchical task network planning instead restricts the search to the decompositions that its methods allow~\cite{erol-et-al-amai1996,nau-et-al-jair2003}, and domain-specific solvers add knowledge of their own, as Rolling Stone does for Sokoban with deadlock tables~\cite{junghanns-schaeffer-aij2001}.
In all of these, a standard procedure still carries out the pruned search, with its open and closed lists, so the space it needs can grow exponentially.
Our policies instead define the search space itself, and the bound on its depth follows from a property of the policy checked once, before any task is seen.

\paragraph{Choice and AND/OR search.}
Golog and other agent programming languages provide nondeterministic choice with backtracking~\cite{levesque-et-al-jlp1997}, and backtracking over variable values is the basic method for constraint satisfaction~\cite{dechter-2003}.
Strong solutions in nondeterministic planning are least fixed points of an AND/OR graph whose $\mathrm{OR}$ vertices are the actions of the agent and whose $\mathrm{AND}$ vertices are the outcomes that the environment picks~\cite{cimatti-et-al-aij2003}.
Our search space has the same shape, but its $\mathrm{AND}$ vertices come from the policy rather than the environment: an ordinary rule must work for every transition it admits, which keeps the rest of the policy search-free.

\paragraph{Relational policies and generalized plans.}
Description-logic concepts and roles were introduced as a feature language for generalized policies~\cite{martin-geffner-ai2004} and later used for rule-based policies whose features are independent of a task's objects~\cite{bonet-geffner-ijcai2018,bonet-et-al-aaai2019,frances-et-al-aaai2021}.
Sketches use the same rules to decompose tasks into subproblems of bounded width~\cite{bonet-geffner-jair2024,drexler-et-al-jair2024}, and indexical policies add registers and operating modes~\cite{bonet-et-al-icaps2024}.
General policies have been learned as logical generalized plans~\cite{srivastava-et-al-aij2011,hu-degiacomo-ijcai2011}, as programs with pointers over objects~\cite{segovia-et-al-aij2019,segovia-et-al-aij2024}, and as neural networks~\cite{toyer-et-al-jair2020,stahlberg-et-al-kr2022,stahlberg-et-al-kr2023}.
All of these select actions without search and so run into the limit of Section~\ref{subsec:completeness}; related arguments show that, unless P${}={}$NP, Blocksworld has unbounded width~\cite{bonet-geffner-jair2024} and no general optimal policy~\cite{stahlberg-et-al-kr2022}.

\paragraph{Language models for planning.}
Language models have been used to synthesize generalized plans and policies, often with automated feedback or soundness checks~\cite{silver-et-al-aaai2024,rossetti-et-al-icaps2024,chen-et-al-icaps2025wslm4plan,stein-et-al-icaps2026,stein-et-al-complete2026}, and have been paired with planners or external validators~\cite{liu-et-al-arxiv2023,kambhampati-et-al-icml2024}, since direct plan generation was long found unreliable~\cite{valmeekam-et-al-neurips2023}.
Recent frontier models, however, rival state-of-the-art planners on tasks based on the IPC 2023 Learning Track~\cite{correa-et-al-arxiv2025}.
They can also synthesize domain-dependent heuristics as Python code, which beat state-of-the-art domain-independent heuristics on test tasks~\cite{correa-et-al-neurips2025b}.
Such a heuristic leaves every decision to best-first search, which can require exponentially many expansions even when its error is a small additive constant~\cite{helmert-roeger-aaai2008}, whereas our policies mark the decisions that need search.
Our learning loop follows the generator--adversary design of previous work~\cite{anonymous-unpublished2026,anonymous-unpublished2026b}, a form of counterexample-guided synthesis~\cite{solar-lezama-et-al-asplos2006} and of abstraction refinement~\cite{clarke-et-al-cav2000,seipp-helmert-jair2018}.

\section{Background}
\label{sec:background}

We follow the notation of \citet{frances-et-al-aaai2021} and \citet{bonet-et-al-icaps2024}. %

A classical planning task is a pair $P=\langle D,I\rangle$ where the domain $D=\tup{\mathcal P,A}$ contains predicates and action schemas, and $I$ gives the objects $O$, the initial state $s_0$, and the goal description $\mathcal G$.
States are sets of ground atoms, applicable ground actions map states to successor states, and a plan is a sequence of ground actions that maps the initial state to a state that satisfies $\mathcal G$.
A task is solvable iff there is a plan for it.
A class of planning tasks is a set $\mathcal Q$ of tasks that share a domain but may vary in objects, initial state, and goal.

For features to compare the current state with the goal, the vocabulary is extended with \emph{goal predicates}~\cite{anonymous-unpublished2026b}.
For each domain predicate $p$, the goal predicate $p_g$ has the same arity, and for every tuple $\bar o$, the atom $p_g(\bar o)$ holds in every state of $P$ iff $\mathcal G$ contains $p(\bar o)$.
A second goal predicate records the atoms that the goal requires to be false.
Every state thus also describes the goal.

Description-logic expressions over the predicates in the vocabulary denote sets of objects (\emph{concepts}) and sets of object pairs (\emph{roles})~\cite{baader-et-al-2003}.
They are built with the usual constructors. %
A \emph{nominal} denotes a single object by name, such as the kitchen in Childsnack, which every task of the class contains.
Emptiness tests over these sets yield Boolean features, and cardinalities yield numerical features.
For example, in Blocksworld, the role of pairs in $\mathit{on}_g$ but not in $\mathit{on}$ relates each block to the block that it must be on but is not yet on, and its cardinality counts the unsatisfied goal atoms.
Distance features measure the length of the shortest path between two concepts along a role, and they take the value $\infty$ if no such path exists.
Numerical features take non-negative integer values that are bounded in each task, as the termination test below requires.
These features refer only to unary and binary predicates.

\subsection{Policies and Structural Termination}

Let $\Phi$ be a vocabulary of concepts, roles, Boolean features, and numerical features, and let $\Psi=\Psi_{\mathrm B}\cup\Psi_{\mathrm N}$ be the sets of Boolean and numerical features in $\Phi$.
A (standard) rule over $\Phi$ has the form $C\mapsto E$, where $C$ is the condition and $E$ is the effect.
The condition $C$ is a set of expressions $b$ and $\neg b$ for $b\in\Psi_{\mathrm B}$, and $n\,{=}\,0$ and $n\,{>}\,0$ for $n\in\Psi_{\mathrm N}$.
The effect $E$ is a set of expressions $b$, $\neg b$, and $b?$ for $b\in\Psi_{\mathrm B}$, and $n\,{\uparrow}$, $n\,{\downarrow}$, and $n?$ for $n\in\Psi_{\mathrm N}$.
Features that $E$ does not mention must keep their values, while $b?$ and $n?$ allow arbitrary changes.
A policy $\Pi$ over $\Phi$ is a finite set of such rules.

A state transition $(s,s')$ is \emph{compatible with} (or satisfies) rule $r$ if $s$ satisfies its condition and the features change across $(s,s')$ according to its effect.
The transition is a $\Pi$-transition if it is compatible with a rule $r$ in $\Pi$.
A $\Pi$-trajectory is a sequence of states in which every consecutive pair is a $\Pi$-transition.
The policy $\Pi$ solves a task $P$ if every maximal $\Pi$-trajectory that starts in the initial state is finite and ends in a goal state.
Equivalently, $\Pi$ is \emph{closed} and \emph{acyclic} in $P$: no maximal $\Pi$-trajectory that starts in the initial state ends in a non-goal state, and no such trajectory is infinite~\cite{frances-et-al-aaai2021}.
This strong solution concept permits an agent to take any $\Pi$-transition.
The policy $\Pi$ solves the class $\mathcal Q$ if it solves every solvable task in $\mathcal Q$.

Indexical policies~\cite{bonet-et-al-icaps2024} extend standard policies with a finite set $M$ of \emph{(operating) modes} and a finite set $\mathsf R$ of \emph{registers}.
The modes implement sequencing and contexts that activate or deactivate rules.
The registers focus the policy on specific objects or object pairs.
Rules of an indexical policy over $\Phi$ are of three sorts~\cite{anonymous-unpublished2026b}:
\begin{itemize}
\item \emph{transition rules} ``$m\rulesep C\mapsto E\rulesep m'$'', where $C\mapsto E$ is a standard rule over $\Phi$ and $m,m'\in M$ are modes;
\item \emph{load rules} ``$m\rulesep C\mapsto\Load(\mathsf C,\mathsf r)\rulesep m'$'', where $C$ is a feature condition, $\mathsf C\in\Phi$ is a concept, and $\mathsf r\in\mathsf R$ is a register; and
\item \emph{do rules} ``$m\rulesep C\mapsto(\op(\mathsf C_1,\ldots,\mathsf C_k),E)\rulesep m'$'', where $C\mapsto E$ is a standard rule over $\Phi$, $\op$ is the label of a $k$-ary action schema, and $\mathsf C_1,\ldots,\mathsf C_k\in\Phi$ are concepts.
\end{itemize}
An \emph{augmented state} is a tuple $\bar s=(s,m,v)$ of a planning state $s$, a mode $m$, and a register valuation $v$.
We refer to its components as $\bar s_\sstate$, $\bar s_\smode$, and $\bar s_\sreg$.
A transition rule changes the planning state according to its standard rule.
A load rule keeps the planning state and stores one object from the denotation of $\mathsf C$ in $\mathsf r$.
A do rule applies a ground action with label $\op$ whose $i$-th argument belongs to the denotation of $\mathsf C_i$, and whose transition satisfies $C\mapsto E$.
All three sorts change the mode from $m$ to $m'$.
A pair $(\bar s,\bar s')$ of augmented states is \emph{compatible} with a rule if the rule's mode is $\bar s_\smode$, its condition holds in $\bar s$, and the rule takes $\bar s$ to $\bar s'$ in one of these three ways~\cite{anonymous-unpublished2026b}.
We also say that the rule \emph{admits} this transition.
Load rules also have a role version, which stores a pair from the denotation of a role $\mathsf R$ in a role register.
The vocabulary contains, for each register $\mathsf r$, a concept or role whose denotation is the content of $\mathsf r$, or the empty set if $\mathsf r$ is empty.
Features may therefore change when a register changes, even if the planning state doesn't.

An indexical policy is a tuple $\Pi=\langle M,\mathsf R,\Phi,m_0,\mathcal R\rangle$, where $m_0$ is the initial mode and $\mathcal R$ is a finite set of indexical rules.
Its executions are trajectories of augmented states that start in $\bar s_0=(s_0,m_0,v_0)$, where $v_0$ leaves every register empty.
As for standard policies, $\Pi$ solves $P$ if $\Pi$ is closed and acyclic in $P$.
Modules and call rules~\cite{bonet-et-al-icaps2024} are left for future work, as in previous work.

A policy is \emph{structurally terminating} if it cannot induce infinite trajectories merely because of the way its features change, independently of their meaning~\cite{srivastava-et-al-aaai2011,drexler-et-al-jair2024,bonet:kr2025}.
For example, the policy $\Pi$ made of just the two rules $m_0\rulesep \set{\neg H,n\,{>}\,0} \mapsto \set{H, n\,{\downarrow}} \rulesep m_0$ and $m_0\rulesep \set{H} \mapsto \set{\neg H} \rulesep m_0$ is structurally terminating as every infinite $\Pi$-trajectory would flip the Boolean $H$ infinitely often, and each flip involves a decrement of the non-negative integer $n$.
Structural termination establishes acyclicity, one of the two conditions for a solution.
Previous work~\cite{anonymous-unpublished2026b} extends structural termination to indexical policies.
Its abstraction has one vertex for each pair of a mode and a qualitative valuation of the Boolean and numerical features, and it ignores the contents of the registers.
A load rule into register $\mathsf r$ may change every feature that depends on $\mathsf r$ arbitrarily.
Its hybrid verifier combines rule elimination with the complete Sieve procedure~\cite{srivastava-et-al-aaai2011} and decides structural termination in this abstraction.
A policy that it accepts is acyclic on every task.

\section{Query and Arithmetic Features}
\label{sec:features}

Description-logic features refer only to unary and binary predicates.
Many planning domains use predicates of higher arity.
In Sokoban, for example, the ternary predicate $\mathit{adj}(l,l',d)$ states that location $l'$ follows location $l$ in direction $d$.
Previous work compiles such predicates away before learning~\cite{anonymous-unpublished2026}.
We instead extend the feature language with queries of relational algebra~\cite{codd-cacm1970,abiteboul-et-al-1995}.

A query denotes a relation over the objects of a task, with one named column for each position.
Atomic queries denote the state atoms or the goal atoms of a predicate of any arity.
Every concept and every role can also be used as a query with one or two columns.
Queries are combined with natural join, projection, renaming of columns, selection by the equality of two columns or with a constant, union, and difference.
The results connect back to the rest of the vocabulary in two ways.
The projection of a query to one or two columns is a concept or a role, and the emptiness and the cardinality of a query are Boolean and numerical features.
Evaluating a query takes polynomial time for a fixed query, since each intermediate relation has at most $n^c$ tuples on a task with $n$ objects, where $c$ is its number of columns.
A query with many columns therefore raises the degree of the bound in Section~\ref{sec:bound}. %

For example, with $\mathit{agent}(l)$ for the player's location, $\mathit{at}(b,l)$ for the location of a box $b$, and $\mathit{clear}(l)$ for an empty location, the query
\[
 \pi_{b,l'}\bigl(\mathit{agent}(r)\bowtie\mathit{at}(b,l)\bowtie\mathit{adj}(r,l,d)\bowtie\mathit{adj}(l,l',d)\bowtie\mathit{clear}(l')\bigr)
\]
denotes pairs $(b,l')$ such that a push in some direction $d$ moves $b$ to $l'$ once the player stands at the location $r$ behind the box.
Queries can express properties that relate more than two objects at once, which description-logic features cannot express directly~\cite{drexler-et-al-kr2024}.

Numerical features are extended with constants, addition, subtraction, multiplication, division, min, and max.
Structural termination asks for a feature that some rule decreases and that no rule increases, and a sum can measure progress that no single count captures: in Spanner, adding the number of loose nuts to the number of locations still reachable from the objects gives a feature that decreases when the man picks up a spanner, walks on, or tightens a nut.
A constant shifts the point at which a feature reaches zero, which matters because a condition compares a numerical feature only with zero: in Peg Solitaire, the number of occupied positions minus two is zero exactly when two pegs are left.
Subtraction is truncated at zero and division rounds down, so numerical features keep non-negative integer values, as the termination test and Theorem~\ref{thm:depth} require.
A value that exceeds a fixed limit, like the distance between two concepts that no path connects, is $\infty$.

\section{Search Spaces and Choose Rules}
\label{sec:choice}

A policy, together with a task, defines a search space over augmented states.
We first define choose rules and search spaces, and then when a policy solves a task as a condition on the search space.
We describe how the search is performed, illustrate choose rules in Sokoban and Peg Solitaire, and define the choice depth and width.

\subsection{Search Spaces}
\label{sec:searchspace}

A \emph{choose rule} has the form ``$m\rulesep C\mapsto\Choose(\mathsf C,\mathsf r)\rulesep m'$''.
In an augmented state $\bar s$ in which the condition $C$ holds, the \emph{candidates} of the rule are the objects in the denotation of $\mathsf C$ in $\bar s$.
Storing a candidate in the register $\mathsf r$ and changing the mode to $m'$ gives a successor of $\bar s$, while the planning state stays the same.
Unlike load rules, choose rules have no role version, because a role can have quadratically many object pairs as candidates.
We call the indexical rules of Section~\ref{sec:background} \emph{ordinary}, in contrast to choose rules.

The search space $G_{\Pi,P}$ induced by policy $\Pi$ on task $P$ is the \emph{bipartite and directed graph} made of augmented-state vertices $\bar s$, and rule vertices $(\bar s, r)$ where $\bar s$ is an augmented state, and $r$ is a policy rule.
A rule ``$m\rulesep C\mapsto\cdots\rulesep m'$'' is \emph{applicable} in $\bar s$ if $m=\bar s_\smode$ and $C$ holds in $\bar s$.
For an ordinary rule ``$r=m\rulesep C\mapsto E\rulesep m'$'', edge $(\bar s, (\bar s, r))$ exists iff there is at least one augmented state $\bar s'$ such that $(\bar s,\bar s')$ is compatible with $r$.
For a choose rule ``$r=m\rulesep C\mapsto\Choose(\mathsf C,\mathsf r)\rulesep m'$'', edge $(\bar s, (\bar s, r))$ exists iff $r$ is applicable in $\bar s$, even if $r$ has no candidate in $\bar s$.
In both cases, edge $((\bar s,r), \bar s')$ exists iff $(\bar s,\bar s')$ is compatible with $r$.
Its root is the initial augmented state $\bar s_0$.

The two sorts of rules differ here because they place different obligations on the policy.
An applicable ordinary rule that admits no transition offers the search nothing, so the policy need not account for it.
A choose rule instead marks a point where the search tries candidates, so it must offer a candidate that works.
If it offers none, its vertex has no children and is unsolved, as defined next.

We make $G_{\Pi,P}$ an \emph{AND/OR graph} by tagging each vertex as AND or OR: every augmented-state vertex $\bar s$ is an AND vertex, every rule vertex $(\bar s,r)$ for an ordinary rule $r$ is an AND vertex, and every rule vertex $(\bar s,r)$ for a choose rule $r$ is an OR vertex.
Tagging every augmented-state vertex as an AND vertex makes the choice among rules universal: if several rules contribute children of $\bar s$, the policy must work for each of them, and an ordinary rule must work for each of its successors.
The candidate of a choose rule is the only existential choice.
This is not the convention of AND/OR search in nondeterministic planning, where the agent chooses the action at an OR vertex and the environment picks the outcome at an AND vertex~\cite{cimatti-et-al-aij2003}.
Making the choices among rules universal instead is what lets an agent follow any transition that the policy allows. %

The \emph{solved} vertices are the least set of vertices that satisfies the conditions below.
An augmented-state vertex $\bar s$ whose planning state $\bar s_\sstate$ satisfies the goal is solved.
An $\mathrm{OR}$ vertex is solved if at least one of its children is solved.
Every other vertex is solved if it has at least one child and all its children are solved.
The three conditions are monotone in the set of solved vertices, so the least such set exists, and it is computed as a least fixed point that starts from the goal states and works backwards.
A vertex is therefore unsolved exactly when nothing in the graph makes it solved: a non-goal augmented state in which no rule admits a transition, an applicable choose rule without candidates, or a vertex whose only support runs through itself, as on a cycle.

\begin{definition}
  \label{def:solved}
  Policy $\Pi$ \emph{solves} task $P$ if the root $\bar s_0$ of $G_{\Pi,P}$ is solved.
  It solves a class $\mathcal Q$ if it solves every solvable task in $\mathcal Q$.
\end{definition}

Definition~\ref{def:solved} does not need structural termination, which only makes the search terminate.
A cycle does not by itself make a choose rule unsolved: a choose rule one of whose candidates leads into a cycle can still be solved through another candidate.

Without choose rules, every rule vertex is an $\mathrm{AND}$ vertex, and Definition~\ref{def:solved} says that every reachable non-goal augmented state has an applicable rule with a successor, and every successor of every applicable rule again leads to a solved state.
These are the closed and acyclic conditions of Section~\ref{sec:background} on the augmented states that $\Pi$ reaches: a maximal trajectory can only end in a goal state, and no reachable cycle consists of non-goal states only, since without choose rules such a cycle is unsolved.
Choose rules relax only the $\mathrm{OR}$ vertices, where a solution discards every candidate but one.

\subsection{Backtracking and Search}

We call the procedure that executes a policy the \emph{scout}.
It follows the policy from the initial augmented state $\bar s_0$, and it returns a plan if it reaches a goal state.

A choose rule turns an ordinary branch into a \emph{backtracking point}: solving no longer needs every candidate to work, only one, and the scout finds that one by search.

In intractable domains, satisfying an $\mathrm{AND}$ vertex in place of a choose rule is not possible in general.
A policy without choose rules must describe, with features alone, a set of candidates none of which leads to a \emph{dead end}, a state from which the goal cannot be reached. %
The scout can then follow any allowed transition and never needs to backtrack.
Policy features are evaluated in polynomial time, and Theorem~\ref{thm:depth} keeps the executions short.
A policy without choose rules therefore decides in polynomial time whether a task has a plan, so it cannot cover a class for which this decision is NP-hard, unless $\text{P}\,{=}\,\text{NP}$ (Proposition~\ref{prop:limits}).
Planning is NP-hard for some benchmark domains~\cite{helmert-aij2003}, and Sokoban is even PSPACE-complete~\cite{culberson-tr1997}.
A choose rule removes this obstacle: only one of its candidates must avoid a dead end, and its $\mathrm{OR}$ vertex is solved once the scout finds it.

To state this formally, we fix the choices.
A \emph{choice function} $\sigma$ selects, for every solved $\mathrm{OR}$ vertex, one solved child that supports it, that is, one that enters the least fixed point before it.
Fixing $\sigma$ at every solved $\mathrm{OR}$ vertex of $G_{\Pi,P}$, keeping every other $\mathrm{AND}$ vertex as it is, removing the children of goal states, since the search stops there, and keeping only the vertices that the root then reaches gives a subgraph $G_{\Pi,P}^\sigma$ in which every vertex has exactly the children it needs to stay solved.
We call $G_{\Pi,P}^\sigma$ a \emph{solution graph:} it exists exactly when the root $\bar s_0$ is solved, and every maximal trajectory in it is finite and ends in a goal state.

A solution graph tells us that a plan exists, but the scout does not know $\sigma$.
Instead, it builds $G_{\Pi,P}$ on demand, one vertex at a time, and explores it with depth-first search.
At an augmented-state vertex, it descends into any one of its children, since Definition~\ref{def:solved} requires all of them to lead to the goal.
It does the same at an $\mathrm{AND}$ rule vertex, and in neither case does it keep an alternative.
At an $\mathrm{OR}$ rule vertex, that is, a choose rule, it tries the candidates one by one, in an arbitrary order, and if the rest of the execution fails, it backtracks to this vertex and tries the next candidate.
The scout thus searches only at $\mathrm{OR}$ vertices.

The search terminates if the policy has no infinite trajectory, that is, if $G_{\Pi,P}$ is acyclic.
Since a choose rule has the same denotation and the same mode change as a load rule with the same arguments, the abstraction of \citet{anonymous-unpublished2026b} can handle choose rules as well.
We call a policy with choose rules \emph{structurally terminating} if it is structurally terminating after replacing its choose rules by load rules.
Such a policy induces an acyclic $G_{\Pi,P}$ on every task.

\begin{propositionrep}
  \label{prop:semantics}
  Let $\Pi$ be a structurally terminating indexical policy that solves a task $P$.
  Then the scout returns a plan for $P$.
\end{propositionrep}

\begin{proof}
  Since $\Pi$ is structurally terminating, $G_{\Pi,P}$ is acyclic.
  Let $\sigma$ be a choice function for which the root of $G_{\Pi,P}$ is solved, and let $G_{\Pi,P}^\sigma$ be the corresponding solution graph.
  Since $G_{\Pi,P}^\sigma$ is acyclic and the set of augmented states is finite, every trajectory of $G_{\Pi,P}^\sigma$ is finite, and since every leaf of $G_{\Pi,P}^\sigma$ is a goal state, every maximal one ends in a goal state.
  For every augmented state $\bar s$ that $G_{\Pi,P}^\sigma$ reaches from $\bar s_0$, let $\ell(\bar s)$ be the length of the longest trajectory of $G_{\Pi,P}^\sigma$ from $\bar s$ to a goal state.
  We show by induction on $\ell(\bar s)$ that the scout finds a plan from every such state $\bar s$.
  If $\bar s$ is a goal state, the scout stops.
  Otherwise, $\bar s$ is solved in $G_{\Pi,P}$, so every applicable choose rule in $\bar s$ has a candidate: an applicable choose rule without one would itself be an unsolved $\mathrm{OR}$ vertex, contradicting that $\bar s$ is solved.
  If the scout takes an ordinary transition, it picks one applicable rule and, if the rule allows several outcomes, one of them; both choices are among children of solved $\mathrm{AND}$ vertices, so the resulting state is itself solved, is reached by $G_{\Pi,P}^\sigma$, and has a smaller value of $\ell$, so the scout finds a plan from it by induction.
  If the scout applies a choose rule, it tries the candidates one by one.
  Each candidate for which the rest of the execution fails is abandoned after finitely many steps, since $G_{\Pi,P}$ is acyclic and every vertex has finitely many children.
  The scout therefore eventually tries the candidate selected by $\sigma$, unless an earlier candidate succeeds, and it finds a plan from the successor of this candidate by induction.
\end{proof}

The converse does not hold.
The scout may find a plan even if $\Pi$ does not solve $P$, since it examines only one child of each $\mathrm{AND}$ vertex.
Structural termination alone ensures only that the scout terminates, either with a plan or with a failure; whether it finds a plan depends on whether $\Pi$ solves $P$.

\begin{example}[Peg Solitaire]\rm
  \label{ex:pegsol}
  A policy that we learned for Peg Solitaire searches over jumps with two choose rules, and has five rules in all.
  Each jump removes one peg, and the tasks ask for one peg to be left, so the numerical feature $r$, the number of occupied positions minus two, is zero exactly when one jump remains.
  The concept $S$ collects the occupied positions from which some jump is possible, and $D$ the positions that a jump can reach from the position in the register $\mathsf s$.
  Both are projections of a query over the ternary predicate $\mathit{in\text{-}line}$, which relates a position, the position jumped over, and the landing position. %
  Writing $P$ for the occupied positions, $G$ for the positions that the goal requires to be occupied, $L$ for the position visited last, $\mathsf C_\mathsf s$ and $\mathsf C_\mathsf d$ for the concepts of the two registers, and $\mathit{jump}$ and $\mathit{end}$ for the two action schemas, the rules are
  \begin{align*}
    m_1 &\rulesep \set{r\,{>}\,0} \mapsto \Choose(S,\mathsf s) \rulesep m_2 \\
    m_2 &\rulesep \emptyset \mapsto \Choose(D,\mathsf d) \rulesep m_3 \\
    m_3 &\rulesep \set{r\,{>}\,0} \mapsto (\mathit{jump}(\mathsf C_\mathsf s,P,\mathsf C_\mathsf d),\set{r\,{\downarrow}}) \rulesep m_4 \\
    m_4 &\rulesep \emptyset \mapsto (\mathit{end}(L),\emptyset) \rulesep m_1 \\
    m_1 &\rulesep \set{r\,{=}\,0} \mapsto (\mathit{jump}(P,P,G),\set{r?}) \rulesep m_5 \,.
  \end{align*}
  The first four rules form a loop: two choose rules select a jump, the third performs it, and the fourth ends the move.
  The third rule decreases $r$ while the other three leave it unchanged, which is what certifies structural termination.
  Its middle argument is the broad concept $P$, yet the two registers already determine the position jumped over, so this $\mathrm{AND}$ vertex has a single child.
  The fifth rule performs the last jump, onto a position that the goal requires to be occupied, and stops in the mode $m_5$.
  The two choose rules are needed because a legal jump can leave the remaining pegs unsolvable, and a candidate concept that excluded every such jump would have to decide solvability.
  Each iteration makes two choices, so an execution makes two choices per jump.
  \qed
\end{example}

\subsection{Choice Depth and Choice Width}
\label{sec:depth}

For a finite trajectory $\tau$, the \emph{choice depth} $d(\tau)$ is the number of $\mathrm{OR}$ vertices on $\tau$ with at least two children.
A choice with a single candidate introduces no alternative and does not count.
The \emph{choice width} $w(\tau)$ is the largest number of children of an $\mathrm{OR}$ vertex on $\tau$, or zero if $\tau$ contains no choose rule.
The choice depth and width, not the number of choose rules, determine how much the scout may have to search.
If every trajectory has choice depth at most $k$ and choice width at most $B$, the scout explores at most $\sum_{i=0}^{k}B^i$ sequences of choices.
Since candidates are objects, the choice width is at most the number $n$ of objects of the task, so the scout explores at most $\sum_{i=0}^{k}n^i$ sequences of choices.
Small choice depth and width are thus natural objectives for learning.

\section{Bounded Search Depth}
\label{sec:bound}

Structurally terminating policies induce an acyclic search space $G_{\Pi,P}$ on every task. %
We now show that the length of every trajectory in $G_{\Pi,P}$ is also bounded.
The bound exploits the fact that numerical features have domains of bounded cardinality.

Let $\kappa$ be the maximum of $1$ and the arities of the relations counted by the numerical features of the policy.
On a task with $n$ objects, a cardinality is at most $(n+1)^\kappa$, and a finite distance is below $n$.
A numerical feature is a fixed expression combining such numbers with non-negative constants and the arithmetic operations of Section~\ref{sec:features}, and its values are non-negative integers or $\infty$.
If an expression contains $m$ occurrences of cardinalities and distances, its finite values are at most $c\,(n+1)^{\kappa m}$ for a constant $c$ that depends only on the expression.
Addition and multiplication respectively add and multiply bounds, while truncated subtraction, integer division, minimum, and maximum cannot produce a finite value above the largest finite operand.
Let $D$ be the largest such $\kappa m$ over the numerical features, or $0$ if there are none.
There are integers $a\geq1$ and $b\geq0$, depending only on the policy, such that $U(n)=a\,(n+1)^D+b$ bounds all finite numerical feature values.
Each numerical feature thus takes at most $p(n)=U(n)+2$ values, including $\infty$.

\begin{toappendix}
\begin{lemma}
  \label{lem:values}
  There are integers $a\geq1$ and $b\geq0$, depending only on $\Pi$, such that on every task with $n$ objects, every finite value of a numerical feature of $\Pi$ is at most $U(n)=a\,(n+1)^D+b$.
  Each numerical feature thus takes at most $p(n)=U(n)+2$ values, including $\infty$.
\end{lemma}

\begin{proof}
  Fix a task $P$ with $n$ objects.
  For each numerical subexpression $e$, we obtain an integer-valued polynomial $U_e(n)\geq1$ that bounds every finite value of $e$ on $P$.
  A constant $c$ has the bound $\max\{1,c\}$, and cardinalities and finite distances have the bound $(n+1)^\kappa$.
  Given bounds $U_1$ and $U_2$ for two operands, their sum has the bound $U_1+U_2$ and their product has the bound $U_1U_2$.
  Truncated subtraction cannot exceed its finite left operand, and division by a positive finite integer cannot exceed its finite numerator.
  A finite minimum or maximum is one of its finite operands.
  If an operand is $\infty$, a finite result of any of these operations is either $0$ or a finite operand, so the same bounds still apply.
  Thus subtraction, division, minimum, and maximum may all use the bound $U_1+U_2$.
  Induction on the expression gives a bound $c\,(n+1)^{\kappa m}$ when it contains $m$ occurrences of cardinalities and distances: addition takes the larger exponent, multiplication adds exponents, and the other operations need no larger exponent than addition.
  Since $\Pi$ has finitely many fixed expressions, $U(n)=a\,(n+1)^D+b$ bounds all finite numerical feature values for suitable $a\geq1$ and $b\geq0$, including those of expressions containing only constants.
  Including $\infty$, each numerical feature has at most $p(n)=U(n)+2$ values.
\end{proof}
\end{toappendix}

\begin{theoremrep}
  \label{thm:depth}
  Let $\Pi$ be a structurally terminating indexical policy with choose rules.
  There is a polynomial $q$ such that, for every task $P$ with $n$ objects, every trajectory of $\Pi$ on $P$ has at most $q(n)$ transitions.
\end{theoremrep}

\begin{proof}
  Fix a task $P$ with $n$ objects.
  For a feature $f\in\Psi$ and an augmented state $\bar s$, write $f(\bar s)$ for its actual value, and define
  \[
    \alpha(\bar s)=\bigl(\bar s_\smode,(f(\bar s))_{f\in\Psi}\bigr).
  \]
  This tuple records numerical values, rather than only whether they are zero or positive.
  Each Boolean feature has at most two values, and each numerical feature has at most $p(n)$ values by Lemma~\ref{lem:values}, so there are at most
  \[
    q(n)=|M|\,2^{|\Psi_{\mathrm B}|}\,p(n)^{|\Psi_{\mathrm N}|}
  \]
  possible tuples $\alpha(\bar s)$.

  Let $G$ be the abstract graph of the termination test~\cite{anonymous-unpublished2026b}, in which choose rules are treated like load rules.
  Its vertices are the pairs $(m,\nu)$ of a mode $m$ and a Boolean valuation $\nu$ of the features in $\Psi$, and its edges are labeled with rules.
  Every transition of $\Pi$ on $P$ satisfies the condition and the effects of a rule that induces it, where load and choose rules allow arbitrary changes to the features that depend on their register and preserve the other features.
  The transition therefore corresponds to an edge of $G$ labeled with this rule.
  If the effect of the rule contains $f\downarrow$ (resp.\ $f\uparrow$), the value of $f$ strictly decreases (resp.\ increases), and if the effect does not mention $f$, the value of $f$ does not change.

  Since $\Pi$ is structurally terminating, the hybrid verifier accepts $\Pi$, and the hybrid verifier accepts exactly the policies that Sieve accepts~\cite{anonymous-unpublished2026b}.
  Sieve repeatedly removes an edge from a strongly connected component $K$ of the current graph if the edge makes progress on a feature $f$ that no edge of $K$ undoes, until the graph is acyclic.
  If $f$ is numerical, the rule of the edge has the effect $f\downarrow$, and no edge of $K$ has the effect $f\uparrow$ or $f?$, or symmetrically with $f\uparrow$.
  If $f$ is Boolean, the rule of the edge flips $f$: its condition requires one value of $f$ and its effect sets the other.
  The removal also requires that no other rule of $K$ may flip $f$ in the opposite direction, where a rule with the effect $f?$ may flip $f$ in both directions and so opposes both.
  The condition of a rule that flips $f$ fixes the value of $f$ before the transition, so all edges of that rule flip $f$ in the same direction, and every edge of $K$ that changes $f$ changes it in the direction of the removed edge.

  Suppose a trajectory repeats a tuple, so that $\alpha(\bar s_i)=\alpha(\bar s_j)$ for some $i<j$.
  For each transition of the segment from $\bar s_i$ to $\bar s_j$, fix a rule that induces it.
  The corresponding edges of $G$ form a closed walk, since the segment starts and ends with the same mode and the same feature values.
  Sieve makes $G$ acyclic, so it removes some edge of this walk; let $e$ be the first one that it removes.
  Until then, no edge of the walk has been removed, so the whole walk lies in the strongly connected component $K$ from which $e$ is removed, and $e$ makes progress on a feature $f$ that no edge of $K$ undoes.
  If $f$ is numerical and $e$ has the effect $f\downarrow$, no edge of the walk has the effect $f\uparrow$ or $f?$, so the value of $f$ never increases along the segment and strictly decreases at $e$, which gives $f(\bar s_j)<f(\bar s_i)$; the case $f\uparrow$ is symmetric.
  If $f$ is Boolean, the segment changes $f$ at $e$ and never in the opposite direction, which gives $f(\bar s_j)\neq f(\bar s_i)$.
  Both cases contradict $\alpha(\bar s_i)=\alpha(\bar s_j)$.

  Hence all tuples along a trajectory are distinct.
  A trajectory with $\ell$ transitions contains $\ell+1$ tuples, so $\ell\leq q(n)-1$.
  The modes, the feature sets, and the polynomial $p$ depend only on $\Pi$, so $q$ is a polynomial in $n$.
\end{proof}

The polynomial $q(\cdot)$ depends only on $\Pi$, through its modes and its features.
Its degree is $D\,|\Psi_{\mathrm N}|$, so the bound is polynomial for a fixed policy, but not uniformly over all policies.

A trajectory of $\Pi$ on $P$ is a root-to-leaf path of $G_{\Pi,P}$ with its rule vertices removed, so a trajectory with $\ell$ transitions corresponds to a path of $G_{\Pi,P}$ with $2\,\ell+1$ vertices, alternating augmented-state and rule vertices and ending at a state vertex.
Theorem~\ref{thm:depth} therefore also bounds the depth of the search space itself: every path of $G_{\Pi,P}$ from the root has at most $2\,q(n)+2$ vertices.

\begin{corollaryrep}
  \label{cor:space}
  Let $\Pi$ be a structurally terminating indexical policy with choose rules, and let $P$ be a task with $n$ objects that $\Pi$ solves.
  Then a depth-first scout that builds $G_{\Pi,P}$ on demand and stores only the current path from the root and the index of the candidate it tries at each $\mathrm{OR}$ vertex on it, returns a plan for $P$ in space polynomial in $n$.
\end{corollaryrep}

\begin{proof}
  The scout enumerates the candidates of each choose rule in a fixed order, so that the index of a candidate determines the candidate.
  It stores the current path of $G_{\Pi,P}$ from the root: an augmented state at each state vertex, and, at each $\mathrm{OR}$ vertex on the path, the index of the candidate it is currently trying.
  At an $\mathrm{AND}$ vertex, whether an augmented state or an ordinary rule, it descends into the first child and stores nothing beyond the child itself.
  If the path reaches a state vertex that is a leaf but not a goal, or an $\mathrm{OR}$ vertex with no candidate left to try, the scout backtracks: it removes the path back to the most recent $\mathrm{OR}$ vertex with an untried candidate, advances its index, and continues from there, or backtracks further if no candidate is left.
  By the bound above, the path never exceeds $2\,q(n)+2$ vertices, and since $G_{\Pi,P}$ is acyclic, the search terminates.
  Since $\Pi$ solves $P$, Proposition~\ref{prop:semantics} shows that it returns a plan, namely the actions along its final trajectory.
  For a fixed domain, a planning state has polynomially many atoms in $n$, a mode has constant size, a register valuation consists of a constant number of objects or object pairs, and a candidate index is a number below $n+1$.
  Each vertex on the stored path thus has polynomial size, and so does the path itself.
  Finally, the features, and hence the successors and the candidates of a state, can be computed in polynomial time and therefore in polynomial space.
\end{proof}

Explicit state-space search, such as breadth-first or greedy best-first search, keeps a set of visited states to avoid cycles and repeated work, and this set may grow exponentially in $n$.
A structurally terminating policy induces an acyclic $G_{\Pi,P}$, so the scout needs no such set.
It stores a single path through $G_{\Pi,P}$, whose length is polynomial by Theorem~\ref{thm:depth}.
Its running time can still be exponential in the choice depth. %
Theorem~\ref{thm:depth} also shows a limitation: a structurally terminating policy produces only plans of polynomial length, so it cannot solve classes that require longer plans.

\subsection{Completeness}
\label{subsec:completeness}

A policy that solves a class $\mathcal Q$ turns the scout into a decision procedure for plan existence on $\mathcal Q$: the scout returns a plan for every task of $\mathcal Q$ that has one (Proposition~\ref{prop:semantics}), and it reports failure on all other tasks.
Its space is polynomial in the number $n$ of objects (Corollary~\ref{cor:space}), and its time is polynomial in $n$ for each of the at most $\sum_{i=0}^{k}B^i$ sequences of choices that it explores.
Since Theorem~\ref{thm:depth} bounds the length of every plan that the scout returns, it also limits which classes such a policy can solve.

\begin{propositionrep}
  \label{prop:limits}
  Let $\Pi$ be a structurally terminating policy that solves class $\mathcal Q$, and let $q$ be the polynomial of Theorem~\ref{thm:depth}.
  Every task of $\mathcal Q$ with $n$ objects that has a plan has one with at most $q(n)$ actions, and deciding whether a task of $\mathcal Q$ has a plan is in NP.
  If the choice depth of $\Pi$ is bounded by a constant on $\mathcal Q$, this decision is in P.
\end{propositionrep}

\begin{proof}
  Let $P\in\mathcal Q$ have $n$ objects and a plan.
  Then $\Pi$ solves $P$, so the root of $G_{\Pi,P}$ is solved and a solution graph $G_{\Pi,P}^\sigma$ exists.
  A maximal trajectory of $G_{\Pi,P}^\sigma$ that starts at $\bar s_0$ ends in a goal state, and it has at most $q(n)$ transitions by Theorem~\ref{thm:depth}.
  The actions that its transitions apply thus form a plan for $P$ with at most $q(n)$ actions.
  Deciding plan existence for $\mathcal Q$ is therefore in NP, since such a plan is a certificate of polynomial size that can be checked in polynomial time for a fixed domain.

  Let now the choice depth be at most a constant $k$ on $\mathcal Q$.
  The scout terminates on every task, because $G_{\Pi,P}$ is acyclic.
  If $P$ has a plan, then $\Pi$ solves $P$ and the scout returns a plan by Proposition~\ref{prop:semantics}.
  Every sequence of actions that the scout returns is a plan, so it reports failure if $P$ has none.
  The scout explores at most $\sum_{i=0}^{k}B^i$ sequences of choices (Section~\ref{sec:depth}), which is polynomial in $n$, because the choice width $B$ is at most $n$.
  It follows each of them for at most $q(n)$ transitions by Theorem~\ref{thm:depth}, and each transition costs polynomial time, since the features are evaluated in polynomial time and a state has polynomially many successors.
  The scout thus decides plan existence for $\mathcal Q$ in polynomial time.
\end{proof}

Three consequences follow.
No structurally terminating policy solves Towers of Hanoi, where the shortest plan for $d$ disks has $2^d-1$ actions, since $q$ is a polynomial.
Unless NP${}={}$PSPACE, none solves Sokoban, which is PSPACE-complete~\cite{culberson-tr1997}, since plans of polynomial length would place it in NP.
Unless P${}={}$NP, none with constant choice depth solves a class with NP-hard plan existence, as in several benchmark domains~\cite{helmert-aij2003}.
The choice depth of a policy thus measures how hard its class can be.
A policy without choose rules has choice depth zero, so it can only solve classes in which plan existence is in P.
Choose rules raise this ceiling from P to NP.

The converse of Proposition~\ref{prop:limits} is open.
We do not know whether every class whose solvable tasks have polynomial-length plans admits a structurally terminating policy that solves it.
Bounded action search does not settle this: structural termination needs progress on the features, not just a bound on the number of steps.

\section{Learning Policies with Choose Rules}
\label{sec:learning}

We learn policies in the counterexample-guided loop of previous work~\cite{anonymous-unpublished2026,anonymous-unpublished2026b}, shown in Figure~\ref{fig:learning-loop}.
Rather than enumerate a pool of features, a language model constructs features and rules on demand, while symbolic tools check each candidate policy and report why it fails.
Two language-model agents drive the loop: a generator, which proposes and revises policies, and an adversary, which proposes tasks on which they fail.
Both receive the PDDL domain, the source of its parameterized task generator, and the syntax and semantics of the feature and policy languages.
The generator source matters because it defines the class of tasks: it shows which initial states and goals can occur, including invariants that the domain description alone does not imply, so a policy may rely on them.

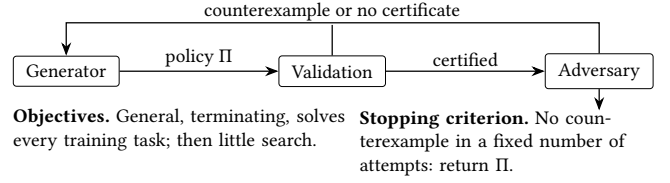
\begin{figure}[t]
\captionsetup{skip=4pt}
\centering
\resizebox{\columnwidth}{!}{%
\begin{tikzpicture}[
  box/.style={draw, rounded corners=1.5pt, inner xsep=6pt, inner ysep=4pt},
  note/.style={align=left, inner sep=0pt},
  lbl/.style={inner sep=2pt},
  arr/.style={-{Stealth[length=2.2mm]}, semithick},
]
  \node[box] (gen) {Generator};
  \node[box, right=26mm of gen] (val) {Validation};
  \node[box, right=26mm of val] (adv) {Adversary};

  \draw[arr] (gen) -- node[lbl, above] {policy $\Pi$} (val);
  \draw[arr] (val) -- node[lbl, above] {certified} (adv);

  \draw[semithick] (val.north) -- ++(0,5mm);
  \draw[arr] (adv.north) -- ++(0,5mm) -| (gen.north);
  \node[lbl, above=5mm of val] {counterexample or no certificate};

  \draw[arr] (adv.south) -- ++(0,-3.5mm);

  \node[note, text width=55mm, below=3.5mm of gen.south west, anchor=north west]
    {\textbf{Objectives.} General, terminating, solves every training task; then little search.};
  \node[note, text width=48mm, below=3.5mm of adv.south east, anchor=north east]
    {\textbf{Stopping criterion.} No counterexample in a fixed number of attempts: return $\Pi$.};
\end{tikzpicture}}
\caption{The learning loop.
The generator proposes a policy $\Pi$ and revises it using returned evidence.
Validation checks structural termination, then every training task by execution and then universal verification.
The adversary proposes task-generator settings and keeps solvable tasks on which $\Pi$ fails.}
\label{fig:learning-loop}
\Description{A cycle of three boxes in a row. The generator sends a policy to validation, validation sends a certified policy to the adversary, and a feedback line above the row returns a counterexample or a missing certificate from validation and from the adversary to the generator. A note under the generator gives its objectives, and an arrow down from the adversary leads to a note with the stopping criterion.}
\end{figure}

\paragraph{Objectives.}
A policy is accepted only if it is structurally terminating and solves every task of the training set, and the prompt asks for a general policy, so that it also solves the tasks outside that set.
Among the policies that qualify, we prefer those that search little, that is, those with a small choice depth.
Replacing a choose rule by a load rule is not a way to get there, since a load rule is an $\mathrm{AND}$ vertex and therefore requires every candidate, not just one, to avoid a dead end.
Instead, the generator can sharpen the candidate concept of a choose rule, add features, or use ordinary rules for phases in which no candidate leads to a dead end; a choice left with a single candidate no longer counts toward the choice depth.
In Sokoban, for example, a feature can recognize a push that moves a box into a corner that is not a goal location, from which the box can never be moved again, and the candidate concept of destinations can then exclude these corners.
Only sharpening can cost a solution, since widening a candidate concept never breaks a policy that works.
With lower priority, the prompt also asks for simple features and for control flow through modes.

\paragraph{Validation.}
Each proposal is checked first for structural termination, with choose rules treated like load rules, and only then on the training tasks, by execution and then by universal verification according to Definition~\ref{def:solved}.
Execution needs to find only one plan, whereas verification confirms that every $\mathrm{AND}$ vertex of the search space has all its children solved and that every $\mathrm{OR}$ vertex has one solved candidate.
Execution is cheaper and often finds failures early, so verification is skipped once it fails.
The generator then receives the failure: a trajectory to where the policy gets stuck, a choose rule without candidates, a cycle, or the part of the abstraction that the termination test cannot certify, together with the search effort, choice depth, and choice width of successful searches.
A missing termination certificate requires a revised policy, while a resource limit in a task check is inconclusive.

\begin{table*}[p]
\newcommand{\MainResultsCombined}{}
\begin{minipage}[t]{\columnwidth}
\ifdefined\GeneratedLearningTableIncluded\else
\global\let\GeneratedLearningTableIncluded\empty
\makeatletter
\ifdefined\MainResultsCombined\else
\begin{table}[p]
\fi
\caption{Learned policies for IPC-learning (top) and Autoscale (bottom).
All indexical policies were proven structurally terminating.
G and A count generator and adversary proposals; T gives total learning time (hours).
P counts distinct learning tasks with proofs of correctness for the retained policy.
F, M, R, and Ch count features, modes, rules, and choose rules.
$d$ and $w$ give maximum successful choice depth and width during validation of the retained policy.
USD estimates token cost at standard short-context API prices (1 October 2026).
In each suite, policies in the lower block use choose rules.
}
\label{tab:learning-results}
\setlength{\tabcolsep}{2pt}
\renewcommand{\arraystretch}{0.9}
\begin{tabular*}{\columnwidth}{@{\extracolsep{\fill}}lrrrrrrrrrrr@{}}
\toprule
Domain & G & A & P & T & F & M & R & Ch & $d$ & $w$ & USD \\
\midrule
\multicolumn{12}{@{}l}{\textbf{(a) IPC-learning}\begingroup\phantomsection\edef\@currentlabel{\thetable a}\label{tab:ipc-learning-policies}\endgroup} \\
\cmidrule[\lightrulewidth](lr){1-12}
Blocksworld & 1 & 4 & 360 & 0.08 & 5 & 4 & 5 & 0 & 0 & 0 & 1.19 \\
Childsnack & 1 & 4 & 240 & 0.78 & 8 & 8 & 9 & 0 & 0 & 0 & 1.64 \\
Ferry & 1 & 4 & 338 & 0.38 & 6 & 5 & 6 & 0 & 0 & 0 & 1.30 \\
Floortile & 1 & 4 & 339 & 0.64 & 7 & 5 & 11 & 0 & 0 & 0 & 1.96 \\
Miconic & 1 & 4 & 360 & 0.35 & 4 & 4 & 8 & 0 & 0 & 0 & 1.07 \\
Rovers & 1 & 4 & 303 & 1.00 & 19 & 17 & 28 & 0 & 0 & 0 & 2.37 \\
Satellite & 1 & 4 & 300 & 1.43 & 8 & 7 & 10 & 0 & 0 & 0 & 1.97 \\
Spanner & 1 & 4 & 305 & 1.82 & 3 & 1 & 3 & 0 & 0 & 0 & 1.39 \\
Transport & 1 & 4 & 360 & 0.50 & 9 & 6 & 8 & 0 & 0 & 0 & 1.56 \\
\cmidrule[\lightrulewidth](lr){1-12}
Sokoban & 15 & 17 & 1331 & 11.05 & 16 & 13 & 18 & 4 & 7 & 7 & 48.22 \\
\cmidrule[\heavyrulewidth](lr){1-12}
\multicolumn{12}{@{}l}{\textbf{(b) Autoscale}\begingroup\phantomsection\edef\@currentlabel{\thetable b}\label{tab:autoscale-policies}\endgroup} \\
\cmidrule[\lightrulewidth](lr){1-12}
Barman & 1 & 4 & 350 & 1.02 & 8 & 16 & 17 & 0 & 0 & 0 & 1.51 \\
Blocksworld & 1 & 4 & 356 & 0.52 & 4 & 4 & 5 & 0 & 0 & 0 & 1.22 \\
Childsnack & 1 & 4 & 300 & 0.89 & 11 & 8 & 10 & 0 & 0 & 0 & 1.79 \\
Data-network & 1 & 4 & 276 & 1.99 & 11 & 8 & 10 & 0 & 0 & 0 & 2.03 \\
Depots & 1 & 4 & 330 & 1.19 & 9 & 7 & 8 & 0 & 0 & 0 & 1.84 \\
Driverlog & 1 & 4 & 331 & 1.12 & 17 & 22 & 31 & 0 & 0 & 0 & 2.05 \\
Elevators & 1 & 4 & 335 & 0.32 & 10 & 8 & 16 & 0 & 0 & 0 & 1.54 \\
Floortile & 1 & 4 & 333 & 0.83 & 9 & 7 & 12 & 0 & 0 & 0 & 2.38 \\
Grid & 1 & 4 & 239 & 1.23 & 9 & 11 & 16 & 0 & 0 & 0 & 2.06 \\
Gripper & 1 & 4 & 9 & 0.06 & 3 & 4 & 4 & 0 & 0 & 0 & 1.19 \\
Hiking & 2 & 5 & 404 & 1.62 & 9 & 16 & 20 & 0 & 0 & 0 & 3.25 \\
Logistics & 1 & 4 & 332 & 0.45 & 11 & 15 & 21 & 0 & 0 & 0 & 1.74 \\
Miconic & 1 & 4 & 300 & 0.32 & 4 & 4 & 8 & 0 & 0 & 0 & 1.32 \\
Parking & 1 & 4 & 298 & 0.64 & 16 & 15 & 29 & 0 & 0 & 0 & 2.24 \\
Rovers & 1 & 4 & 304 & 1.11 & 13 & 13 & 22 & 0 & 0 & 0 & 2.37 \\
Satellite & 1 & 4 & 306 & 0.47 & 11 & 8 & 12 & 0 & 0 & 0 & 1.65 \\
Scanalyzer & 1 & 4 & 9 & 0.27 & 18 & 15 & 25 & 0 & 0 & 0 & 2.33 \\
Storage & 1 & 4 & 277 & 2.03 & 13 & 5 & 9 & 0 & 0 & 0 & 2.46 \\
Tidybot & 6 & 10 & 560 & 11.25 & 32 & 38 & 88 & 0 & 0 & 0 & 28.66 \\
TPP & 1 & 4 & 354 & 0.34 & 11 & 8 & 10 & 0 & 0 & 0 & 1.59 \\
Transport & 1 & 4 & 338 & 0.65 & 9 & 6 & 8 & 0 & 0 & 0 & 1.46 \\
Visitall & 1 & 4 & 262 & 1.00 & 3 & 1 & 2 & 0 & 0 & 0 & 1.17 \\
Zenotravel & 1 & 4 & 320 & 0.46 & 9 & 9 & 14 & 0 & 0 & 0 & 1.79 \\
\cmidrule[\lightrulewidth](lr){1-12}
Agricola & 1 & 4 & 235 & 3.19 & 40 & 21 & 37 & 2 & 51 & 11 & 4.08 \\
Freecell & 12 & 13 & 984 & 4.80 & 30 & 21 & 57 & 9 & 1072 & 10 & 35.62 \\
Mprime & 13 & 15 & 848 & 4.87 & 36 & 27 & 43 & 10 & 29 & 10 & 46.26 \\
Nomystery & 1 & 4 & 326 & 1.24 & 8 & 5 & 9 & 2 & 11 & 5 & 2.45 \\
Pegsol & 1 & 4 & 335 & 0.94 & 8 & 5 & 5 & 2 & 29 & 16 & 1.32 \\
Snake & 11 & 12 & 852 & 6.10 & 21 & 20 & 61 & 8 & 689 & 24 & 29.78 \\
Sokoban & 20 & 10 & 892 & 9.56 & 64 & 125 & 182 & 50 & 35 & 14 & 72.56 \\
Termes & 1 & 4 & 335 & 0.88 & 13 & 10 & 23 & 3 & 29 & 7 & 2.90 \\
Tetris & 2 & 4 & 337 & 1.68 & 6 & 4 & 27 & 3 & 34 & 24 & 3.99 \\
Woodworking & 1 & 4 & 296 & 2.96 & 15 & 6 & 16 & 1 & 7 & 5 & 2.83 \\
\bottomrule
\end{tabular*}
\par
\ifdefined\MainResultsCombined\else
\end{table}
\fi
\makeatother
\fi

\end{minipage}\hfill
\begin{minipage}[t]{\columnwidth}
\input{tables/test_results}
\end{minipage}
\end{table*}

\begin{figure*}[t]
\input{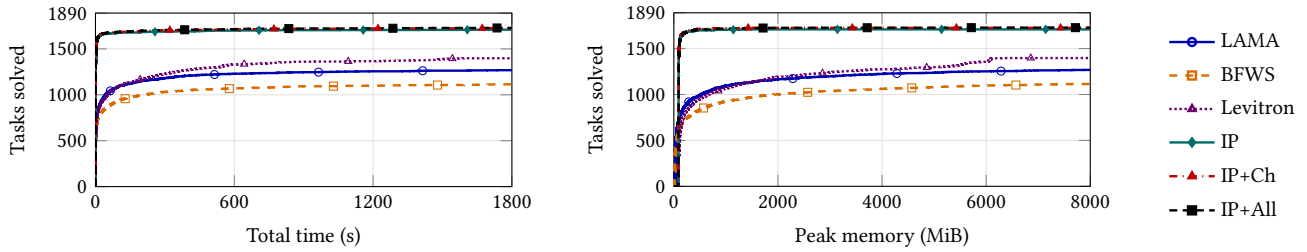}
\caption{Solved tasks with total wall time (left) or peak memory (right) at most the horizontal-axis value, across 900 IPC-learning and 990 Autoscale tasks (linear axes).
IP+Ch stores explored choice points; IP+All stores all explored augmented states.}
\label{fig:cumulative-resources}
\Description{Two cumulative step plots show solved-task counts against total time from 0 to 1800 seconds and peak memory from 0 to 8000 MiB, with six curves for LAMA, BFWS, Levitron, IP, IP+Ch, and IP+All, combining IPC-learning and Autoscale Agile.}
\end{figure*}

\paragraph{Adversary.}
The adversary proposes task-generator settings intended to expose weaknesses, for example the number of boxes and the size of the layout in Sokoban.
A task enters the training set only if a planner finds a plan for it, which establishes solvability independently of the candidate policy, and only if the policy then fails on it or reaches a resource limit.
Earlier tasks are rechecked before each new request, so that a revision which breaks one of them is detected.
The adversary sees the outcome of every setting that it has tried, which helps it select tasks that challenge the policy while remaining small enough for the symbolic checks.

Returning a policy does not imply that every retained task has been verified, since checks that remain inconclusive are reported as such.
Proposal and repair budgets bound the learning process, and exhausting them can end a run without a returned policy.

\section{Experiments}

We address three questions.
Does the learning loop produce structurally terminating policies that use choose rules where plan existence is hard?
Do the learned policies generalize to the test tasks, and how do they compare with classical planners in coverage, time, and memory?
How much do the learned policies search, and what does storing the explored choice points gain?

\subsection{Setup}

\paragraph{Benchmarks.}
Our first benchmark, IPC-learning, is the test set of the IPC 2023 Learning Track~\cite{taitler-et-al-aimag2024}.
Each domain has 30 easy, 30 medium, and 30 hard test tasks, which gives 900 tasks in total.
Policies are learned from generated training tasks, and the test tasks are used only for evaluation.
The second benchmark, Autoscale, comprises 33 domains from the Autoscale Agile suite~\cite{torralba-et-al-icaps2021} for which PDDL task generators are available.
It contains 30 tasks per domain and 990 tasks in total.
Training tasks are much smaller than test tasks: outside Sokoban and Nomystery, no training task has more than 65 objects, whereas test tasks have up to 19,615 objects.
Peg Solitaire is the exception, where every task has 33 objects.
The supplementary material lists the sizes per domain.

\paragraph{Learning.}
Both agents use \texttt{gpt-6-astra} with \texttt{max} reasoning effort.
We run the loop once per domain, and all runs finish within 20 outer iterations.
Each generator call allows five proposals, each with five syntax-repair attempts; each adversary call allows three proposals.
Each adversary proposal specifies three configurations with 30 seeds each.
Execution, verification, and solvability checks each allow 60~seconds and $10^6$ states per task; execution uses one rollout.
Structural termination checks allow 12 relevant features per cyclic component of the mode graph after preprocessing.

\paragraph{Evaluation.}
Each test task uses one thread, with limits of 30~minutes and 8000~MiB.
We compare with LAMA~\cite{richter-westphal-jair2010}, BFWS~\cite{lipovetzky-geffner-aaai2017}, and Levitron~\cite{correa-et-al-ipc2023a}.
IP runs the learned policies with the scout of Section~\ref{sec:choice}, which stores only its current path.
Two variants of the scout spend space to save time.
IP+Ch also stores every \emph{choice point} that it has explored, that is, every augmented state of an $\mathrm{OR}$ vertex, so that it never searches a failed subtree again.
IP+All stores every explored augmented state.
The stored set can grow exponentially, so both variants give up the space bound of Corollary~\ref{cor:space}, but IP+Ch stores far fewer states than IP+All.

\appto{\appendixprelim}{\nobalance\raggedbottom}
\begin{toappendix}
Table~\ref{tab:task-sizes} compares the number of objects of the training and test tasks.
Table~\ref{tab:ablation-learning-results} reports learning effort and policy sizes for the ablation without queries.
Table~\ref{tab:ablation-test-results} compares test coverage for the main and ablation policies under IP, IP+Ch, and IP+All, together with LAMA, BFWS, and Levitron.

Without queries, the learned policies use choose rules in four additional domains: Rovers in both suites, and Data-network, Tidybot, and TPP in Autoscale.
All four domains are tractable, and the policies learned with queries avoid choose rules in each of them.
The additional choose rules thus introduce unnecessary search, indicating insufficient feature expressivity without queries.
The ablation also has two unsuccessful learning runs, both in Autoscale: Mprime reaches the iteration limit, and Sokoban is cancelled after roughly 24 hours.
In both cases, we test the last available policy before learning stops.

\begin{table}[!htbp]
\caption{Number of objects of the training and test tasks.
Training tasks are all tasks that the adversary retained during learning, and $n$ counts them.
Med and max give the median and the maximum number of objects per task.}
\label{tab:task-sizes}
\begin{tabular*}{\columnwidth}{@{\extracolsep{\fill}}lrrrrr@{}}
\toprule
& \multicolumn{3}{c}{Training} & \multicolumn{2}{c}{Test} \\
\cmidrule(lr){2-4}\cmidrule(l){5-6}
Domain & $n$ & Med & Max & Med & Max \\
\midrule
\multicolumn{6}{@{}l}{\textbf{(a) IPC-learning}} \\
\cmidrule[\lightrulewidth](lr){1-6}
Blocksworld & 360 & 11 & 27 & 90 & 488 \\
Childsnack & 360 & 25 & 28 & 126 & 1326 \\
Ferry & 360 & 19 & 26 & 88 & 1461 \\
Floortile & 339 & 16 & 22 & 234 & 1022 \\
Miconic & 360 & 14 & 24 & 94 & 681 \\
Rovers & 360 & 21 & 30 & 149 & 596 \\
Satellite & 360 & 18 & 23 & 100 & 402 \\
Sokoban & 1331 & 127 & 578 & 1210 & 9880 \\
Spanner & 335 & 20 & 27 & 124 & 833 \\
Transport & 360 & 18 & 26 & 74 & 354 \\
\cmidrule[\lightrulewidth](lr){1-6}
\multicolumn{6}{@{}l}{\textbf{(b) Autoscale}} \\
\cmidrule[\lightrulewidth](lr){1-6}
Agricola & 328 & 29 & 48 & 54 & 76 \\
Barman & 350 & 21 & 31 & 84 & 158 \\
Blocksworld & 359 & 16 & 40 & 80 & 151 \\
Childsnack & 360 & 24 & 28 & 104 & 143 \\
Data-network & 276 & 31 & 36 & 192 & 262 \\
Depots & 353 & 21 & 25 & 256 & 477 \\
Driverlog & 360 & 23 & 42 & 294 & 542 \\
Elevators & 360 & 20 & 32 & 120 & 214 \\
Floortile & 344 & 17 & 27 & 500 & 1726 \\
Freecell & 1041 & 38 & 65 & 80 & 113 \\
Grid & 360 & 16 & 22 & 5418 & 19615 \\
Gripper & 12 & 13 & 21 & 96 & 169 \\
Hiking & 418 & 18 & 43 & 30 & 44 \\
Logistics & 360 & 18 & 28 & 1014 & 2376 \\
Miconic & 360 & 17 & 26 & 154 & 279 \\
Mprime & 848 & 18 & 24 & 62 & 98 \\
Nomystery & 356 & 24 & 78 & 640 & 981 \\
Parking & 346 & 13 & 19 & 74 & 127 \\
Pegsol & 352 & 33 & 33 & 33 & 33 \\
Rovers & 360 & 24 & 34 & 48 & 148 \\
Satellite & 360 & 19 & 32 & 264 & 507 \\
Scanalyzer & 10 & 16 & 20 & 74 & 132 \\
Snake & 852 & 20 & 27 & 89 & 182 \\
Sokoban & 898 & 59 & 91 & 113 & 170 \\
Storage & 341 & 24 & 32 & 132 & 300 \\
Termes & 335 & 11 & 17 & 53 & 107 \\
Tetris & 341 & 31 & 56 & 124 & 212 \\
Tidybot & 560 & 29 & 35 & 40 & 56 \\
TPP & 360 & 16 & 20 & 69 & 124 \\
Transport & 338 & 23 & 32 & 122 & 255 \\
Visitall & 360 & 26 & 36 & 2552 & 8836 \\
Woodworking & 345 & 36 & 46 & 172 & 238 \\
Zenotravel & 360 & 19 & 22 & 125 & 217 \\
\bottomrule
\end{tabular*}
\end{table}

\makeatletter
\begin{table}[!htbp]
\caption{Learning without queries on IPC-learning (top) and Autoscale (bottom).
All finalized policies were proven structurally terminating.
Columns follow Table~\ref{tab:learning-results}, with T in hours.
$\dagger$ marks interrupted learning: Mprime reached the iteration limit, and Autoscale Sokoban was cancelled after roughly 24 hours.
Available partial statistics are shown; -- denotes unavailable statistics.
In each suite, policies in the lower block use choose rules.}
\label{tab:ablation-learning-results}
\setlength{\tabcolsep}{2pt}
\renewcommand{\arraystretch}{0.9}
\begin{tabular*}{\columnwidth}{@{\extracolsep{\fill}}lrrrrrrrrrrr@{}}
\toprule
Domain & G & A & P & T & F & M & R & Ch & $d$ & $w$ & USD \\
\midrule
\multicolumn{12}{@{}l}{\textbf{(a) IPC-learning}\begingroup\phantomsection\edef\@currentlabel{\thetable a}\label{tab:ipc-learning-ablation-policies}\endgroup} \\
\cmidrule[\lightrulewidth](lr){1-12}
Blocksworld & 1 & 4 & 359 & 0.29 & 5 & 4 & 5 & 0 & 0 & 0 & 1.08 \\
Childsnack & 1 & 4 & 240 & 0.83 & 10 & 7 & 9 & 0 & 0 & 0 & 1.52 \\
Ferry & 1 & 4 & 336 & 0.35 & 4 & 4 & 5 & 0 & 0 & 0 & 1.12 \\
Floortile & 1 & 4 & 329 & 0.77 & 7 & 5 & 11 & 0 & 0 & 0 & 1.49 \\
Miconic & 1 & 4 & 360 & 0.32 & 4 & 4 & 8 & 0 & 0 & 0 & 1.29 \\
Satellite & 1 & 4 & 288 & 0.74 & 11 & 9 & 13 & 0 & 0 & 0 & 1.41 \\
Spanner & 1 & 4 & 330 & 0.88 & 4 & 2 & 4 & 0 & 0 & 0 & 1.22 \\
Transport & 1 & 4 & 360 & 0.39 & 9 & 6 & 8 & 0 & 0 & 0 & 1.28 \\
\cmidrule[\lightrulewidth](lr){1-12}
Rovers & 4 & 7 & 612 & 2.28 & 37 & 51 & 91 & 15 & 160 & 19 & 13.72 \\
Sokoban & 5 & 8 & 643 & 5.82 & 59 & 62 & 94 & 12 & 138 & 6 & 15.04 \\
\cmidrule[\heavyrulewidth](lr){1-12}
\multicolumn{12}{@{}l}{\textbf{(b) Autoscale}\begingroup\phantomsection\edef\@currentlabel{\thetable b}\label{tab:autoscale-ablation-policies}\endgroup} \\
\cmidrule[\lightrulewidth](lr){1-12}
Barman & 1 & 4 & 352 & 0.90 & 9 & 17 & 18 & 0 & 0 & 0 & 1.56 \\
Blocksworld & 1 & 4 & 359 & 0.12 & 5 & 4 & 5 & 0 & 0 & 0 & 1.40 \\
Childsnack & 1 & 4 & 270 & 0.52 & 11 & 8 & 10 & 0 & 0 & 0 & 1.45 \\
Depots & 1 & 4 & 292 & 1.12 & 8 & 7 & 8 & 0 & 0 & 0 & 1.41 \\
Driverlog & 1 & 4 & 355 & 0.70 & 13 & 21 & 32 & 0 & 0 & 0 & 2.91 \\
Elevators & 1 & 4 & 356 & 0.29 & 10 & 7 & 15 & 0 & 0 & 0 & 1.76 \\
Floortile & 1 & 4 & 345 & 0.65 & 8 & 5 & 10 & 0 & 0 & 0 & 1.73 \\
Grid & 1 & 4 & 274 & 1.25 & 8 & 8 & 13 & 0 & 0 & 0 & 2.11 \\
Gripper & 1 & 4 & 8 & 0.09 & 3 & 5 & 5 & 0 & 0 & 0 & 1.13 \\
Hiking & 2 & 5 & 450 & 0.85 & 16 & 21 & 28 & 0 & 0 & 0 & 3.49 \\
Logistics & 1 & 4 & 334 & 0.54 & 9 & 14 & 23 & 0 & 0 & 0 & 1.97 \\
Miconic & 1 & 4 & 360 & 0.19 & 4 & 4 & 8 & 0 & 0 & 0 & 1.32 \\
Mprime$^\dagger$ & 20 & 4 & -- & 5.32 & -- & -- & -- & -- & -- & -- & 49.48 \\
Parking & 1 & 4 & 218 & 0.92 & 19 & 17 & 33 & 0 & 0 & 0 & 2.06 \\
Satellite & 1 & 4 & 298 & 0.58 & 11 & 8 & 11 & 0 & 0 & 0 & 2.14 \\
Scanalyzer & 1 & 4 & 7 & 0.40 & 13 & 34 & 49 & 0 & 0 & 0 & 2.37 \\
Storage & 1 & 4 & 300 & 1.42 & 14 & 9 & 15 & 0 & 0 & 0 & 2.53 \\
Transport & 1 & 4 & 331 & 0.84 & 10 & 6 & 8 & 0 & 0 & 0 & 1.34 \\
Visitall & 1 & 4 & 278 & 0.83 & 3 & 1 & 2 & 0 & 0 & 0 & 1.00 \\
Zenotravel & 1 & 4 & 353 & 0.39 & 11 & 11 & 17 & 0 & 0 & 0 & 1.57 \\
\cmidrule[\lightrulewidth](lr){1-12}
Agricola & 1 & 4 & 319 & 2.22 & 26 & 12 & 33 & 3 & 45 & 14 & 4.70 \\
Data-network & 1 & 4 & 293 & 1.97 & 19 & 15 & 25 & 2 & 8 & 7 & 3.47 \\
Freecell & 13 & 14 & 1097 & 4.04 & 24 & 26 & 112 & 14 & 134 & 9 & 39.98 \\
Nomystery & 4 & 6 & 508 & 2.20 & 9 & 95 & 187 & 23 & 17 & 23 & 10.88 \\
Pegsol & 1 & 4 & 308 & 1.04 & 10 & 6 & 7 & 3 & 43 & 24 & 1.92 \\
Rovers & 4 & 6 & 540 & 1.20 & 22 & 45 & 102 & 18 & 46 & 383 & 8.92 \\
Snake & 5 & 6 & 501 & 1.82 & 7 & 5 & 10 & 4 & 163 & 23 & 8.96 \\
Sokoban$^\dagger$ & 17 & 10 & -- & -- & -- & -- & -- & -- & -- & -- & -- \\
Termes & 1 & 4 & 331 & 1.31 & 15 & 10 & 23 & 2 & 45 & 8 & 3.48 \\
Tetris & 19 & 10 & 786 & 17.45 & 41 & 17 & 71 & 1 & 65 & 30 & 52.17 \\
Tidybot & 6 & 9 & 480 & 12.45 & 56 & 79 & 157 & 4 & 29 & 11 & 24.53 \\
TPP & 1 & 4 & 360 & 0.36 & 14 & 10 & 15 & 1 & 15 & 25 & 2.45 \\
Woodworking & 1 & 4 & 295 & 2.83 & 16 & 9 & 27 & 2 & 8 & 8 & 3.33 \\
\bottomrule
\end{tabular*}
\end{table}
\makeatother

\makeatletter
\begin{table*}[!tp]
\caption{Testing the main policies and the ablation without queries on IPC-learning (top) and Autoscale (bottom).
For ablation runs ending in failure or timeout (Autoscale Mprime and Sokoban), we test the last available policy.
S counts solved tasks; bold marks maximum coverage, including ties.
IP stores no visited states, IP+Ch stores explored choice points, and IP+All stores every explored augmented state.
Both $d,w$ pairs give maxima over tasks solved by all six policy configurations, independently of the baselines.
Within each family, all three modes have identical maxima on these tasks; -- denotes an empty common solved set.
The domain blocks follow the main policies: policies in the lower block use choose rules.}
\label{tab:ablation-test-results}
\setlength{\tabcolsep}{2pt}
\renewcommand{\arraystretch}{0.82}
\begin{tabular*}{\textwidth}{@{\extracolsep{\fill}}l*{13}{c}@{}}
\toprule
 & \multicolumn{3}{c}{Baselines} & \multicolumn{5}{c}{Main} & \multicolumn{5}{c}{Ablation} \\
\cmidrule(lr){2-4}\cmidrule(lr){5-9}\cmidrule(lr){10-14}
Domain & LAMA & BFWS & Levitron & IP & IP+Ch & IP+All & \multicolumn{2}{c}{Shared} & IP & IP+Ch & IP+All & \multicolumn{2}{c}{Shared} \\
\cmidrule(lr){2-2}\cmidrule(lr){3-3}\cmidrule(lr){4-4}\cmidrule(lr){5-5}\cmidrule(lr){6-6}\cmidrule(lr){7-7}\cmidrule(lr){8-9}\cmidrule(lr){10-10}\cmidrule(lr){11-11}\cmidrule(lr){12-12}\cmidrule(lr){13-14}
 & S & S & S & S & S & S & $d$ & $w$ & S & S & S & $d$ & $w$ \\
\midrule
\multicolumn{14}{@{}l}{\textbf{(a) IPC-learning (90 tasks/domain)}\begingroup\phantomsection\edef\@currentlabel{\thetable a}\label{tab:ipc-learning-ablation-coverage}\endgroup} \\
\cmidrule[\lightrulewidth](lr){1-14}
Blocksworld & 61 & 36 & 61 & \textbf{90} & \textbf{90} & \textbf{90} & 0 & 0 & \textbf{90} & \textbf{90} & \textbf{90} & 0 & 0 \\ %
Childsnack & 35 & 35 & 54 & \textbf{90} & \textbf{90} & \textbf{90} & 0 & 0 & \textbf{90} & \textbf{90} & \textbf{90} & 0 & 0 \\ %
Ferry & 69 & 66 & 63 & \textbf{90} & \textbf{90} & \textbf{90} & 0 & 0 & \textbf{90} & \textbf{90} & \textbf{90} & 0 & 0 \\ %
Floortile & 11 & 13 & 38 & \textbf{90} & \textbf{90} & \textbf{90} & 0 & 0 & \textbf{90} & \textbf{90} & \textbf{90} & 0 & 0 \\ %
Miconic & \textbf{90} & \textbf{90} & \textbf{90} & \textbf{90} & \textbf{90} & \textbf{90} & 0 & 0 & \textbf{90} & \textbf{90} & \textbf{90} & 0 & 0 \\ %
Rovers & 70 & 65 & 67 & \textbf{90} & \textbf{90} & \textbf{90} & 0 & 0 & \textbf{90} & \textbf{90} & \textbf{90} & 2309 & 313 \\ %
Satellite & 89 & 65 & \textbf{90} & \textbf{90} & \textbf{90} & \textbf{90} & 0 & 0 & \textbf{90} & \textbf{90} & \textbf{90} & 0 & 0 \\ %
Spanner & 30 & 30 & 59 & \textbf{90} & \textbf{90} & \textbf{90} & 0 & 0 & \textbf{90} & \textbf{90} & \textbf{90} & 0 & 0 \\ %
Transport & 66 & 63 & 64 & \textbf{90} & \textbf{90} & \textbf{90} & 0 & 0 & \textbf{90} & \textbf{90} & \textbf{90} & 0 & 0 \\ %
\cmidrule[\lightrulewidth](lr){1-14}
Subtotal (810) & 521 & 463 & 586 & \textbf{810} & \textbf{810} & \textbf{810} &  &  & \textbf{810} & \textbf{810} & \textbf{810} &  &  \\
\cmidrule[\lightrulewidth](lr){1-14}
Sokoban & 41 & \textbf{46} & 44 & 40 & 41 & 42 & 7 & 7 & 30 & 41 & 43 & 215 & 7 \\ %
\cmidrule[\lightrulewidth](lr){1-14}
Total (900) & 562 & 509 & 630 & 850 & 851 & 852 &  &  & 840 & 851 & \textbf{853} &  &  \\
\cmidrule[\heavyrulewidth](lr){1-14}
\multicolumn{14}{@{}l}{\textbf{(b) Autoscale (30 tasks/domain)}\begingroup\phantomsection\edef\@currentlabel{\thetable b}\label{tab:autoscale-ablation-coverage}\endgroup} \\
\cmidrule[\lightrulewidth](lr){1-14}
Barman & 26 & 8 & 28 & \textbf{30} & \textbf{30} & \textbf{30} & 0 & 0 & \textbf{30} & \textbf{30} & \textbf{30} & 0 & 0 \\ %
Blocksworld & 27 & 10 & 29 & \textbf{30} & \textbf{30} & \textbf{30} & 0 & 0 & \textbf{30} & \textbf{30} & \textbf{30} & 0 & 0 \\ %
Childsnack & 9 & 7 & 21 & \textbf{30} & \textbf{30} & \textbf{30} & 0 & 0 & \textbf{30} & \textbf{30} & \textbf{30} & 0 & 0 \\ %
Data-network & 20 & 13 & 22 & \textbf{30} & \textbf{30} & \textbf{30} & 0 & 0 & \textbf{30} & \textbf{30} & \textbf{30} & 74 & 47 \\ %
Depots & 22 & 21 & 17 & \textbf{30} & \textbf{30} & \textbf{30} & 0 & 0 & \textbf{30} & \textbf{30} & \textbf{30} & 0 & 0 \\ %
Driverlog & \textbf{30} & \textbf{30} & 28 & \textbf{30} & \textbf{30} & \textbf{30} & 0 & 0 & \textbf{30} & \textbf{30} & \textbf{30} & 0 & 0 \\ %
Elevators & \textbf{30} & \textbf{30} & \textbf{30} & \textbf{30} & \textbf{30} & \textbf{30} & 0 & 0 & \textbf{30} & \textbf{30} & \textbf{30} & 0 & 0 \\ %
Floortile & 2 & 2 & 8 & \textbf{30} & \textbf{30} & \textbf{30} & 0 & 0 & \textbf{30} & \textbf{30} & \textbf{30} & 0 & 0 \\ %
Grid & \textbf{20} & 10 & \textbf{20} & 16 & 16 & 15 & 0 & 0 & 13 & 13 & 13 & 0 & 0 \\ %
Gripper & \textbf{30} & \textbf{30} & \textbf{30} & \textbf{30} & \textbf{30} & \textbf{30} & 0 & 0 & \textbf{30} & \textbf{30} & \textbf{30} & 0 & 0 \\ %
Hiking & 23 & 2 & 21 & \textbf{30} & \textbf{30} & \textbf{30} & 0 & 0 & \textbf{30} & \textbf{30} & \textbf{30} & 0 & 0 \\ %
Logistics & 18 & 10 & 15 & \textbf{30} & \textbf{30} & \textbf{30} & 0 & 0 & \textbf{30} & \textbf{30} & \textbf{30} & 0 & 0 \\ %
Miconic & \textbf{30} & \textbf{30} & \textbf{30} & \textbf{30} & \textbf{30} & \textbf{30} & 0 & 0 & \textbf{30} & \textbf{30} & \textbf{30} & 0 & 0 \\ %
Parking & 28 & 25 & 21 & \textbf{30} & \textbf{30} & \textbf{30} & 0 & 0 & \textbf{30} & \textbf{30} & \textbf{30} & 0 & 0 \\ %
Rovers & \textbf{30} & \textbf{30} & \textbf{30} & \textbf{30} & \textbf{30} & \textbf{30} & 0 & 0 & \textbf{30} & \textbf{30} & \textbf{30} & 140 & 29 \\ %
Satellite & 24 & 13 & \textbf{30} & \textbf{30} & \textbf{30} & \textbf{30} & 0 & 0 & \textbf{30} & \textbf{30} & \textbf{30} & 0 & 0 \\ %
Scanalyzer & 24 & 21 & \textbf{30} & \textbf{30} & \textbf{30} & \textbf{30} & 0 & 0 & \textbf{30} & \textbf{30} & \textbf{30} & 0 & 0 \\ %
Storage & 16 & 26 & 17 & \textbf{30} & \textbf{30} & \textbf{30} & 0 & 0 & \textbf{30} & \textbf{30} & \textbf{30} & 0 & 0 \\ %
Tidybot & 17 & 13 & 22 & \textbf{30} & \textbf{30} & \textbf{30} & 0 & 0 & \textbf{30} & \textbf{30} & \textbf{30} & 65 & 13 \\ %
TPP & 25 & 12 & 24 & \textbf{30} & \textbf{30} & \textbf{30} & 0 & 0 & \textbf{30} & \textbf{30} & \textbf{30} & 28 & 55 \\ %
Transport & 12 & 20 & 12 & \textbf{30} & \textbf{30} & \textbf{30} & 0 & 0 & \textbf{30} & \textbf{30} & \textbf{30} & 0 & 0 \\ %
Visitall & \textbf{30} & 25 & \textbf{30} & \textbf{30} & \textbf{30} & \textbf{30} & 0 & 0 & \textbf{30} & \textbf{30} & \textbf{30} & 0 & 0 \\ %
Zenotravel & 20 & 16 & 16 & \textbf{30} & \textbf{30} & \textbf{30} & 0 & 0 & \textbf{30} & \textbf{30} & \textbf{30} & 0 & 0 \\ %
\cmidrule[\lightrulewidth](lr){1-14}
Subtotal (690) & 513 & 404 & 531 & \textbf{676} & \textbf{676} & 675 &  &  & 673 & 673 & 673 &  &  \\
\cmidrule[\lightrulewidth](lr){1-14}
Agricola & \textbf{30} & \textbf{30} & \textbf{30} & \textbf{30} & \textbf{30} & \textbf{30} & 49 & 9 & \textbf{30} & \textbf{30} & \textbf{30} & 72 & 12 \\ %
Freecell & 28 & \textbf{30} & 29 & 10 & 12 & 14 & 582 & 12 & 14 & 22 & 24 & 78 & 10 \\ %
Mprime & 9 & 12 & 24 & 27 & \textbf{28} & \textbf{28} & 687 & 25 & 21 & 21 & 21 & 0 & 0 \\ %
Nomystery & 10 & 14 & \textbf{22} & 1 & 5 & 5 & 28 & 4 & 1 & 2 & 2 & 29 & 4 \\ %
Pegsol & \textbf{30} & \textbf{30} & \textbf{30} & 27 & \textbf{30} & \textbf{30} & 17 & 11 & 27 & 29 & 29 & 29 & 17 \\ %
Snake & 9 & 18 & 16 & 23 & 23 & 23 & 11530 & 15 & 28 & \textbf{29} & \textbf{29} & 1398 & 63 \\ %
Sokoban & 22 & 24 & \textbf{27} & 3 & 3 & 3 & -- & -- & 0 & 3 & 3 & -- & -- \\ %
Termes & \textbf{22} & 12 & 20 & 6 & 10 & 10 & 60 & 17 & 7 & 10 & 10 & 62 & 17 \\ %
Tetris & 19 & 26 & 23 & 26 & \textbf{28} & \textbf{28} & 28 & 96 & 15 & 17 & 17 & 90 & 33 \\ %
Woodworking & 14 & 5 & 16 & \textbf{30} & \textbf{30} & \textbf{30} & 119 & 47 & 19 & 20 & 20 & 55 & 59 \\ %
\cmidrule[\lightrulewidth](lr){1-14}
Subtotal (300) & 193 & 201 & \textbf{237} & 183 & 199 & 201 &  &  & 162 & 183 & 185 &  &  \\
\cmidrule[\lightrulewidth](lr){1-14}
Total (990) & 706 & 605 & 768 & 859 & 875 & \textbf{876} &  &  & 835 & 856 & 858 &  &  \\
\bottomrule
\end{tabular*}
\end{table*}
\makeatother

\end{toappendix}

\subsection{Results}

\paragraph{Learning.}
Table~\ref{tab:learning-results} shows that the learned policies avoid choose rules in domains where plan existence is in P, such as Blocksworld, Logistics, Rovers, and Satellite~\cite{helmert-aij2003,helmert-icaps2006}.
They use choose rules where plan existence is NP-complete, as in Freecell and Mprime~\cite{helmert-aij2003}, or PSPACE-complete, as in Sokoban~\cite{culberson-tr1997}.
Most policies without choose rules need one generator proposal (USD 1.07--2.46); Freecell, Mprime, Snake, and Sokoban need 11--20 generator proposals (USD 29.78--72.56).
Every policy is proven structurally terminating and proven to solve its training tasks (column P).
The ablation in the supplementary material removes queries and introduces choose rules in Tidybot, Rovers, Data-network, and TPP, all of which are tractable and contain predicates of arity greater than two.
Queries provide access to these predicates, avoiding unnecessary search caused by limited feature expressivity.

\paragraph{Testing.}
Table~\ref{tab:test-results} shows that IP solves 850 of 900 IPC-learning tasks and 859 of 990 Autoscale tasks, compared with 630 and 768 for Levitron, the strongest baseline overall.
Without choose rules, the scout never backtracks, and IP solves all 810 IPC-learning tasks and 676 of 690 Autoscale tasks.
In the Autoscale domains with choose rules, IP matches or exceeds the best baseline in 5 of 10 domains, despite choice depths up to $11{,}530$, choice widths up to 100, and unordered candidates.
It falls far behind in Freecell, Nomystery, Sokoban, and Termes.
PSPACE-complete Sokoban is outside our scope: unless NP${}={}$PSPACE, no structurally terminating policy solves all its tasks (Section~\ref{subsec:completeness}).
In Figure~\ref{fig:cumulative-resources}, IP solves most tasks within one second and 100~MiB, never exceeding 1,400~MiB; baselines keep gaining coverage up to 30~minutes and 8000~MiB.
This reflects the space bound of Corollary~\ref{cor:space}: the scout stores only its current path.
IP+Ch solves 1 and 16 more tasks than IP, with peak memory within 2\% of that of IP on 90\% of tasks and at most 291~MiB higher.
IP+All solves 852 and 876 tasks, one more per benchmark than IP+Ch, but uses up to 1,819~MiB more than IP.

\section{Conclusions}

We set out to learn search control for intractable planning domains, rather than policies that decide a plan outright.
Choose rules mark exactly the decisions that need search, and structural termination bounds every trajectory by a polynomial in the size of the task.
A depth-first scout therefore finds a plan in polynomial space, at a cost in time that is exponential only in the choice depth.

We learned such policies with the counterexample-guided loop of previous work, extended with choose rules and with relational-algebra queries.
With them, the scout solves more test tasks than LAMA, BFWS, and Levitron on both benchmarks, most within one second and 100~MiB.
Storing explored choice points solves a few more tasks at little extra memory, but gives up the space bound.

Our guarantees have limits.
Structural termination bounds the search but does not guarantee that it succeeds: a learned policy is proven to solve only its training tasks, and it can still fail on a task outside them.
There is also a hard limit, since the bound that gives polynomial space caps the length of the plans that such a policy can produce.
A class that such a policy solves therefore lies in NP, which excludes Towers of Hanoi outright, and Sokoban unless NP${}={}$PSPACE.
A natural next step is to learn policies with modules and calls, and to extend the depth bound to them.

\restoresection
\bibliographystyle{ACM-Reference-Format}
\bibliography{abbrv,literatur,crossref,extra,followup}
\end{document}